\documentclass{article}	

    \PassOptionsToPackage{numbers, compress}{natbib}

\usepackage[preprint]{neurips_2021}

\usepackage[utf8]{inputenc} 
\usepackage[T1]{fontenc}    
\usepackage{hyperref}       
\usepackage{url}            
\usepackage{booktabs}       
\usepackage{amsfonts}       
\usepackage{amsmath}
\usepackage{amssymb}
\usepackage{amsthm}
\usepackage{graphicx}
\usepackage{txfonts}
\usepackage{wrapfig}
\usepackage{pifont}
\usepackage{lipsum}
\usepackage{nicefrac}       
\usepackage{microtype}      
\usepackage{xcolor}         
\usepackage[ruled,linesnumbered,vlined]{algorithm2e}
\usepackage{algorithmic}
\usepackage{multirow}
\usepackage{siunitx}
\usepackage{bbding}

\usepackage{float}
\newtheorem{lemma}{Lemma}
\newtheorem{theorem}{Theorem}
\newtheorem{definition}{Definition}
\usepackage{array}
\usepackage{diagbox}
\newcolumntype{L}[1]{>{\raggedright\let\newline\\\arraybackslash\hspace{0pt}}m{#1}}
\newcolumntype{C}[1]{>{\centering\let\newline\\\arraybackslash\hspace{0pt}}m{#1}}
\newcolumntype{R}[1]{>{\raggedleft\let\newline\\\arraybackslash\hspace{0pt}}m{#1}}
\SetKwInput{KwInput}{Input}                
\SetKwInput{KwOutput}{Output}              

\title{Learning on Abstract Domains: A New Approach for Verifiable Guarantee in Reinforcement Learning}

\author{%
	Peng Jin, Min Zhang, Jianwen Li, Li Han\\
	Shanghai Key Laboratory of Trustworthy Computing\\
	East China Normal University\\
	\texttt{zhangmin@sei.ecnu.edu.cn} \\
	 \And
	 Xuejun Wen\\
	 Huawei Singapore Research Center\\
	 Huawei Technologies Co., Ltd.\\
	 \texttt{wen.xuejun@huawei.com}
}

\begin{document}
	
	\maketitle
	
	\begin{abstract}
		
		Formally verifying Deep Reinforcement Learning (DRL) systems is a challenging task due to the dynamic continuity of system behaviors and the black-box feature of embedded neural networks. In this paper, we propose a novel abstraction-based approach to train DRL systems on finite abstract domains instead of concrete system states. It yields neural networks whose input states are finite, making hosting DRL systems directly verifiable using model checking techniques. 
		Our approach is orthogonal to existing DRL algorithms and off-the-shelf model checkers. We implement a resulting prototype training and verification framework and conduct extensive experiments on the state-of-the-art benchmark. The results show that the systems trained in our approach can be verified more efficiently while they retain comparable performance against those that are trained without abstraction.

	\end{abstract}
	
	\section{Introduction}
	Despite the unparalleled potential that Deep Reinforcement Learning (DRL) techniques have exposed in plentiful control fields~\cite{DBLP:journals/nature/MnihKSRVBGRFOPB15,DBLP:journals/corr/abs-1810-03259,DBLP:journals/ral/LambertDYLCP19}, real-world DRL applications are quite limited in safety-critical domains because they need certificates for their reliability.
	A typical example is the fully autonomous driving, which is still argued a long way off due to safety concerns~\cite{gomes2016will}. Verifiable guarantees on safety and reliability are both desirable and necessary to those DRL systems~\cite{hasanbeig2020towards}. Unfortunately, formally verifying DRL systems is a challenging task due to the dynamic continuity of system behaviors and the black-box feature of the AI models (neural networks) embedded in the systems.   
	The dynamic continuity results in uncountably-infinite state space~\cite{DBLP:conf/pldi/ZhuXMJ19}, while the black-box feature causes unexplainability of neural networks~\cite{DBLP:journals/csr/HuangKRSSTWY20}. 

	Instead of directly verifying DRL systems, most of the existing approaches rely on  
	transforming them into verifiable models. Representative works include 
	exacting decision trees \cite{bastani2018verifiable} and 
	programmatic policies \cite{DBLP:conf/icml/VermaMSKC18},  
	synthesizing deterministic programs \cite{DBLP:conf/pldi/ZhuXMJ19} and linear controllers \cite{xiong2021scalable}, transforming into hybrid systems \cite{DBLP:conf/hybrid/IvanovWAPL19} and star sets \cite{tran2019safety}. 
	Although these transformation-based approaches are effective solutions, there are some limitations, e.g., extracted policies may not equivalently represent source neural networks 
	and the properties that can be verified may be limited. 
	For instance, only safety properties are supported by hybrid system and star set models by reachability analysis. Thus, it is desired that a trained DRL system can be directly and efficiently verified without transformation.

	In this paper, we propose a novel training approach for DRL by learning on finite abstract domains, unlike the traditional approaches which learn directly on concrete system states. 
	Specifically, we discretize continuous states into finite abstract states, on which we train a DRL system. For the finiteness of abstract states, the neural network trained on them is essentially a finite function that maps from  abstract states to  actions. Because the trained neural network adopts the same action for the concrete states of the same abstract state, we can leverage the abstract interpretation technique \cite{cousot1977abstract} to model the DRL system as a finite-state transition system, which can be efficiently model checked. 

	Our training approach has two main features that distinguish itself from existing classic DRL approaches.  
	Firstly, a DRL system trained in our approach is directly verifiable, and thus it avoids any shortage of  transformation-based approaches. The novelty of learning on abstract domains makes it possible to 
	model a DRL system into a finite-state system by abstracting continuous concrete states into corresponding abstract domains. The following verification becomes straightforward as a bunch of off-the-shelf model-checking tools such as Spot~\cite{duret2016spot} can be used to verify various properties efficiently. Secondly, our approach is orthogonal to existing DRL algorithms and can be naturally implemented by extending them. 
	We have implemented a resulting prototype training and verification framework, and performed extensive  experiments on four classic continuous control tasks. The experimental results demonstrate that the systems trained in our approach have comparable performance with those trained by existing DRL algorithms. Meanwhile, they can be formally verified against desired properties.  
	
	In summary, this paper makes the following two major contributions:
	\begin{enumerate}
		\item A novel abstraction-based DRL approach to train continuous control systems on abstract domains such that the trained systems are amenable to formal verification while retaining comparable performance to those trained on concrete states.
		\item A subsequent abstraction-based verification approach and a resulting prototype tool for model checking the trained DRL systems, coupled with a benchmark of four verified DRL systems for corresponding classic control problems. 
	\end{enumerate}
	
	%

	
	\section{DRL and its Formal Verification}
	\label{sec:motivation}
	DRL is usually modeled as a Markov Decision Process (MDP)~\cite{feinberg2012handbook}, which is a 4-tuple $(S, A, T_a, R_a)$,  where $S$ is a set of states called the state space, $A$ is a set of actions called the action space, $T_a(s, s{'})$ is the probability of the transition from  $s$ to $s{'}$ based on action $a$, and  $R_a(s, s{'})$ is the reward received by the controller after the given transition from $s$ to $s'$. Since the system dynamics of safety-critical systems are generally known and deterministic~\cite{bastani2018verifiable,DBLP:conf/hybrid/IvanovWAPL19,DBLP:conf/pldi/ZhuXMJ19}, it implies that the effect of an action $a$ on a state results in only one successor state. Thus, we write $(s,s')\in T_a$ to indicate that there is a transition from $s$ to $s'$ due to action $a$. 
	
	DRL aims to train a DNN-based controller to learn a deterministic policy $\pi: S \rightarrow A$ that specifies a unique action adopted in a state to achieve specific goals. A trained DRL system can be represented as a tuple $M=(S,A,T_a,\pi, S^{0})$ with $S^{0}$ being a set of initial states of the system. Let $RS$ be the set of all the reachable states of $M$. We have $S^{0} \subseteq RS$, and for two states $s,s'\in S$, if $s\in RS$ and $(s,s')\in T_{\pi(s)}$ then there is $s'\in RS$.  
	
	\begin{wrapfigure}{r}{7.5cm}
		\vspace{-5mm}
		\setlength{\textfloatsep}{5pt}
		\centering
		\includegraphics[width=0.5\columnwidth]{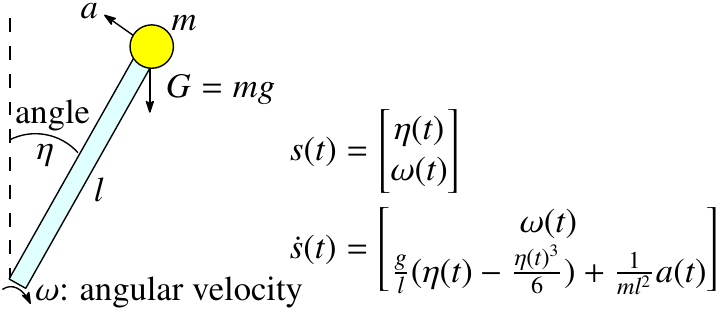}
		\vspace{-1mm}
		\caption{State transition of Inverted Pendulum \cite{DBLP:conf/pldi/ZhuXMJ19}}
		\label{fig:invert}
		\vspace{-3mm}
	\end{wrapfigure}
	
	The formal verification of a DRL system $M$ is to check whether $M$ satisfies some desired properties that are formalized as logical formulas $\phi$ in some logic such as Linear Temporal Logic (LTL)~\cite{pnueli1977temporal}. $M$ satisfies $\phi$, denoted as $M\models \phi$, if and only if all the paths of $M$ satisfy $\phi$. There are two key factors make it intractable to directly verify $M\models \phi$. One is that the number of paths of $M$ is infinite when $S^{0}$ contains infinite states. The other is that the set of successor states is difficult to compute and represent due to the non-linearity of the system dynamics. Figure~\ref{fig:invert} shows an example of computing successor state of the state $s(t)$ using the change of rate $\dot{s}$, where time is discretized into time interval $t$ and the transition from time $kt$ to $(k+1)t$ 
	is approximated by the equation   $s(kt+t)=s(kt)+\dot{s}(kt)\times t$ \cite{DBLP:conf/pldi/ZhuXMJ19}. Further, 
	it needs to compute the control action $a$ by  querying DNN in every transition in order to build the state transition system of a DRL system, which drastically reduces the efficiency of verification.

	Perturbation is another factor making the verification problem of DRL systems more difficult. A trained controller may face perturbations in the real world or caused by modeling errors and differences in training and test scenarios ~\cite{DBLP:conf/icml/TesslerEM19,zhang2020robust}. 
	It is necessary to ensure the robustness of DRL system, so perturbations must be taken into account to verify system robustness. Perturbations may cause nondeterministic transitions between states because the actual successor state may deviate from the expected state due to perturbation~\cite{lutjens2020certified}. 
	We use the perturbation vector $\epsilon$ to describe the offset range.
	Then the target system for verification can be modeled as a tuple $M = (S, A, T_{a}^{\epsilon}, \pi, I)$, where $S$, $A$, $\pi$ and $I$ are the same as previously defined. Given the perturbation vector $\epsilon$, $T_{a}^{\epsilon}: S \rightarrow 2^{S}$ denotes all reachable states after applying $\pi(s)$ to the states $s$. Specifically, for the expected transition from $s^t$ to $s^{t+1}$, actual reachable states $T_{a}^{\epsilon}(s^t)=\{s | \forall i\in\{1,\ldots,n\}.|s_{i} - s^{t+1}_{i} | \leq \epsilon_{i}  \}$, where $n$ is the dimension of the state. Apparently, perturbation to concrete states may lead to state space exploration.

	\begin{figure}[tb]
		\centering
		\includegraphics[width=0.9\columnwidth]{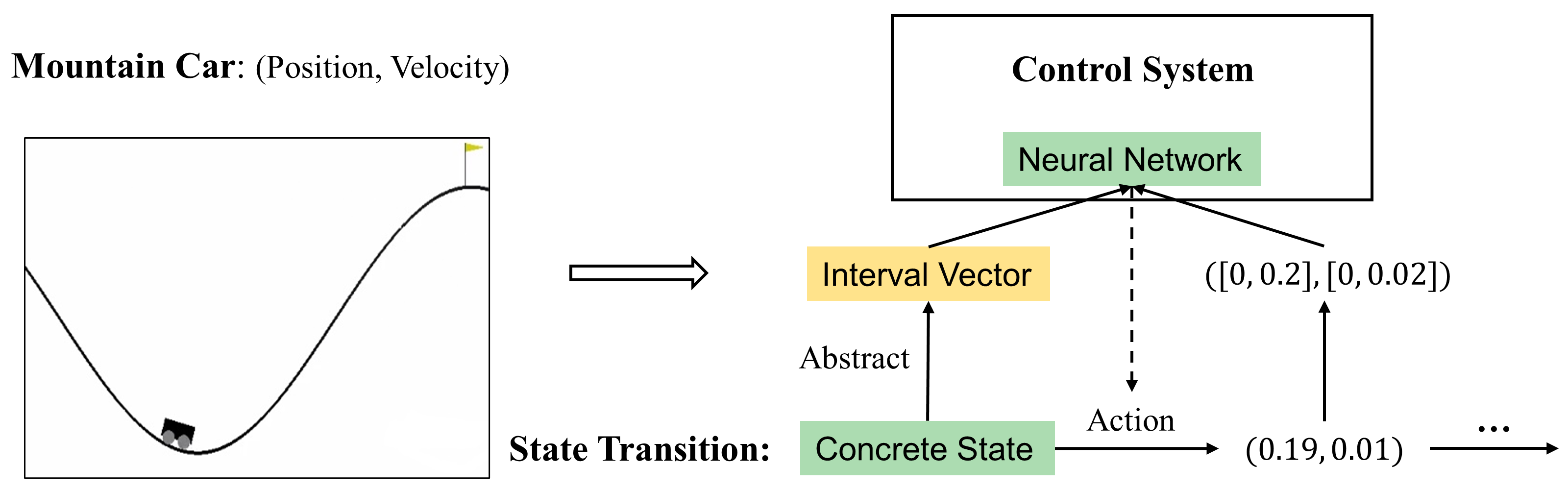}		
		\caption{Abstraction-based deep reinforcement learning}
		\vspace{-3mm}
		\label{framework}
	\end{figure}
	
	\vspace{-1mm}
	\section{Abstraction-Based Reinforcement Learning}
	\label{sec:tt}
	\vspace{-1mm}
	
	Figure~\ref{framework} shows the framework with an illustrative example. The state of the mountain car is a pair of position and velocity. We suppose a region where the position is in $[0,0.2]$ and the velocity is in $[0,0.02]$. Before a concrete state, e.g., $(0.19,0.01)$, is fed to the neural network, we transform it into the representation of its corresponding region, i.e., the interval vector $([0, 0.2], [0, 0.02])$, as the actual input. The neural network produces an action based on its current setting and the input. The action takes effect on the concrete state to drive the system under training to proceed. 
	
	The essential difference of our framework from classic DRL approaches is that the states fed into neural networks are \emph{abstract states}.
	An abstract state corresponds to an infinite set of concrete states, and is represented as a vector of intervals in our framework.  Thus, we call our learning approach abstraction-based reinforcement learning.

	\vspace{-1mm}
	
	\subsection{State Discretization and Abstraction}
	\vspace{-1mm}
	Our abstraction mechanism is based on the assumption that a trained controller usually adopts the same action for those concrete states that are adjacent~\cite{DBLP:conf/formats/Bacci020}. We consider a concrete state $s$ to be a vector of $n$ ($n\geq 1$) real numbers. The distance between two states can be measured by $L_p$ norms. 
	
	\begin{definition}[Adjacent states]
		Two states $s,s'$ are called adjacent with respect to an $L_p$-norm distance $\delta$, denoted by $s\approx s'$,  if and only if $\left\|s-s'\right\|_{p}\leq \delta$.  
	\end{definition}
	
	Given a state $s$ and an $L_p$-norm distance $\delta$, the set of all the adjacent states of $s$ is essentially an $L$-norm ball $B(s,\delta)=\{s'|\left\|s-s'\right\|_{p}\leq \delta\}$.

	
	Let $L_i,U_i$ be the lower and upper bounds for the $i$-th dimension element $s_i$ in $s$. Then the state space ${S}_n$ of the control system is $\Pi^n_{i=1}[L_i,U_i]$. 
	The basic idea of state discretization and  abstraction is to classify all adjacent states into a set, and represent the set as an abstract domain such as Polyhedra, Octagon, and Interval \cite{singh2017practical}. 
	
	In our abstraction approach, we choose Interval as the abstract domain for its simplicity and efficiency. 
	Specifically, we divide the interval $[L_i,U_i]$ of each dimension into a finite set of unit intervals. For each dimension, let $d_i\in \mathbb{R}$ ($0<d_i\leq U_i-L_i$) be the diameter of each unit interval, $d$ represent the vector of diameters for the $n$ dimensions. We call $d$ the \emph{abstraction granularity} of ${S}_n$, and use  ${\cal I}_i=[L_i,U_i]/{d_i}$ to represent the set of all the divided unit intervals. 
	Then, we obtain an abstract-state space ${\cal S}_{n,d}={\cal I}_1\times\ldots\times {\cal I}_n$, where 
	an abstract state $\mathbf{s}$ is essentially a vector of $n$ unit intervals $(I_1,I_2,\ldots,I_n)$. Apparently, ${\cal S}_{n,d}$ is finite. A concrete state $s$ belongs to the abstract state $\mathbf{s}$, denoted by $s\in \mathbf{s}$, if and only if $l_i\leq s_i \leq u_i$ for each $I_i=[l_i,u_i]$.  
	
	\begin{definition}[Interval-based abstraction]
		Given a state space ${S}_n$ and an abstraction granularity $d$, a state $s\in {S}_n$ is abstracted to be an interval vector $(I_1,I_2\ldots,I_n)\in {\cal S}_{n,d}$ where  $l_i\leq s_i \leq u_i$ for each $I_i=[l_i,u_i]$ with $i = 1,\ldots,n$. 
		
	\end{definition}

	\subsection{Learning on Abstract States}
	
	The abstraction-based reinforcement learning approach is orthogonal to most of the state-of-the-art DRL algorithms and can be smoothly implemented atop them. 
	We only need to insert an abstract transformer between 
	the control system and the neural network to transform 
	concrete states into abstract ones before feeding them to the neural network.

	\begin{wrapfigure}{R}{0.575\textwidth}
		\vspace{-4mm}
		\IncMargin{10pt}
		\begin{minipage}{0.575\textwidth}
			\setlength{\textfloatsep}{5pt}
			\begin{algorithm}[H]
				\caption{Abstraction-Based Deep Q-Learning} 
				\label{alg:TDQL}
				
				\For{episode = 1, $M$}{
					Initialize $s^1$ after resetting the Environment\\
					$\mathbf{s}^1 = \mathit{abstractionMapping}(s^1, d)$\\
					\For{t = 1, $T$}{
						Take action $a^t$ based on $Q$($\epsilon$-greedy)\\
						Execute $a^t$, then observe $r^t$ and $s^{t+1}$\\
						$\mathbf{s}^{t+1} = \mathit{abstractionMapping}(s^{t+1}, d)$\\
						Store $(\mathbf{s}^t, a^t, r^t, \mathbf{s}^{t+1})$ in Buffer\\
						Sample batch $(\mathbf{s}^i, a^i, r^i, \mathbf{s}^{i+1})$ from Buffer\\
						Update parameters based on \textit{Loss Function}\\
					}
				}		
			\end{algorithm}
		\end{minipage}
	\vspace{-2mm}
	\end{wrapfigure}
	
	We consider incorporating the operation to extend Deep Q-Learning (DQL)~\cite{DBLP:journals/corr/MnihKSGAWR13} as an illustrative example. Algorithm~\ref{alg:TDQL} depicts the main workflow, where \textit{abstractionMapping} is an abstraction function that maps concrete states to their corresponding abstract states and $d$ is the abstraction granularity. In our abstraction approach, 
	given a concrete state $s=(s_1,\ldots,s_n)$, 
	we first compute the unit interval $[l_i, u_i]$ according to the preset abstraction granularity $d$ such that $l_i\leq s_i\leq u_i$ with $i=1,\ldots,n$. 
	Then the interval vector $\mathbf{s}=([l_1, u_1],\ldots,[l_n,u_n])$ is fed into neural network. 
	It is worth mentioning that we need to double the input dimension of the neural network in order to accept the interval vector. We omit explanations of other steps as they are well-established in DQL. 
	
	We also applied the abstraction technique on Deep Deterministic Policy Gradient (DDPG)~\cite{DBLP:journals/corr/LillicrapHPHETS15} and Proximal Policy Optimization (PPO)~\cite{DBLP:journals/corr/SchulmanWDRK17} algorithms, then conducted experiments using the extended learning algorithm based on the open-sourced DRL library TF2RL~\cite{ota2020tf2rl}, where various DRL algorithms are implemented using TensorFlow 2.x.

	Abstraction plays a crucial role in our framework. Its granularity $d$ determines the performance of a trained network and the verification difficulty of the hosting system. The finer the abstraction is, the better performance a trained network is of, while the more costly the verification becomes due to state space explosion. 
	This assertion is confirmed by the experimental results in Section~\ref{sub:ag}. Therefore, it is important to determine an appropriate abstraction granularity to reach a trade-off between the performance and verification cost. We set $d$ as a hyperparameter in training algorithms, which means the adjustment to it depends on the corresponding training performance.
	
	\vspace{-1mm}
	\section{Abstraction-Based Formal Verification}
	\vspace{-1mm}
	\label{sec:sv}
	In this section, we propose an abstraction-based verification approach to model check the DRL systems trained on abstract domains. The basic idea of our approach is based on the Abstract Interpretation technique~\cite{cousot1977abstract}, which builds transition systems on finite abstract-state spaces by transforming concrete states into abstract ones for the purpose of model checking. Because the abstract state space is finite, its verification can be achieved by classic model-checking techniques~\cite{DBLP:journals/fac/Konnov19}.

	\subsection{Building Abstract-State Transition System}
	\label{subsec:abs}

	We abstract a continuous state space into a finite abstract-state space in the same way as we do in the training phase, and then build an abstract-state transition system by establishing the transition relations among abstract states according to the actions produced by the trained neural network.

	As mentioned in Section~\ref{sec:motivation}, a trained DRL system can be modeled as $M = (S, A, T_{a}^{\epsilon}, \pi, S^0)$ when perturbation is considered. Here, $\pi$ is a neural network that can be modeled as a black-box function $\pi:{\cal S}_{n,d}\rightarrow A$. 
	Let ${\cal S}_{n,d}^0\subseteq {\cal S}_{n,d}$ be a set of abstract states such that $\mathbf{s} \in {\cal S}_{n,d}^0$ if and only if there exists a state $s\in S^0$ such that $s\in \mathbf{s}$.

	Next we define the relation between abstract states. 
	Figure~\ref{fig:fsts} depicts the abstract transformer for abstract states. Given a abstract state $\mathbf{s}^0$, we can obtain a unique action $a$ by feeding it to the trained network. After applying $a$ to $\mathbf{s}^0$, we calculate the interval vector $V$ to cover the irregular state space generated by $\mathbf{s}^0$.
	 If the situation with perturbation is considered, $V$ can be smoothly expanded to $V^{'}$ to include extra reachable states. Then we use a set of abstract states to over approximate $V^{'}$. Let $[l'_i,u'_i]$ be the $i$-th interval in the vector $V^{'}$, then $[l'_i,u'_i]$ must be the sub-interval of either a unit interval or the concatenation of multiple unit intervals of $[l_i,u_i]$. 	Without loss of generality, we assume at least $m_i$ ($m_i\geq 1$) unit interval(s) is (are) needed to concatenate each other to cover $[l'_i,u'_i]$. So we need $\Pi^{n}_{i=1}m_i$ abstract states, whose union is the least over-approximation of the resulting vector. There is a transition relation from $\mathbf{s}^0$ to each abstract state in the union e.g., $\mathbf{s}^{1},\ldots,\mathbf{s}^{4}$ in the figure. 
	 
	 \begin{wrapfigure}{r}{0.45\textwidth}
	 	\setlength{\textfloatsep}{5pt}
	 	\vspace{-4mm}
	 	\centering
	 	\includegraphics[width=0.35\textwidth]{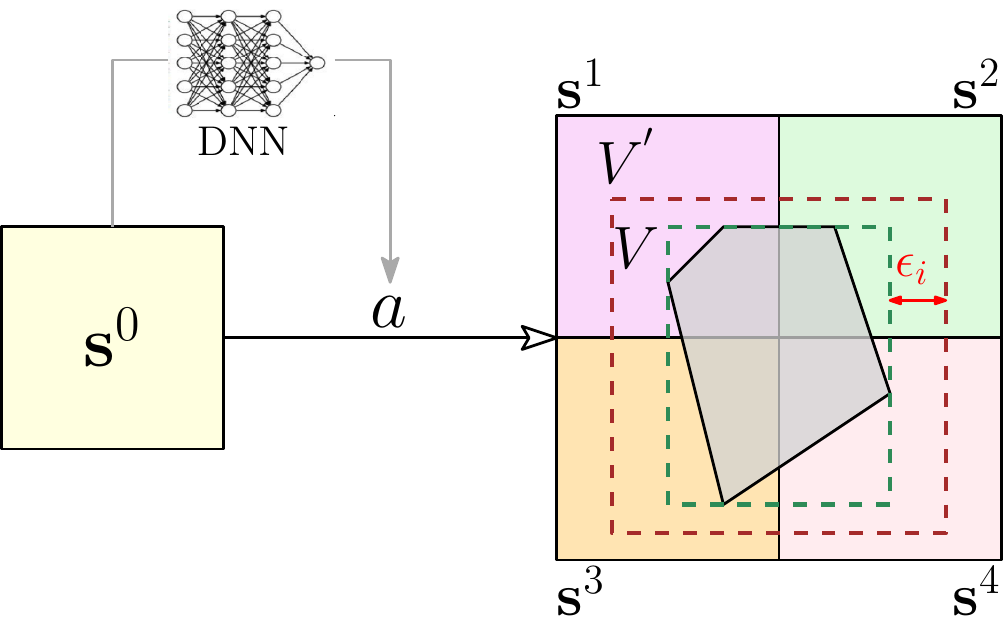}
	 	\caption{Transitions between abstract states}
	 	\label{fig:fsts}
	 	\vspace{-3mm}
	 \end{wrapfigure}

	Let us consider an example of the mountain car shown in Figure \ref{framework}. We assume that the current abstract state is $[0, 0.2] \times [0, 0.02]$. The trained DNN takes the same action $a$ on all the concrete states that are represented by the abstract state.  We assume that action is to accelerate the car to the right side. We calculate the maximal and minimal values on 2 dimensions based on the system dynamics. Then we construct an interval vector with them to represent all states transited from those in the preceding abstract state. We assume the vector is $[0.11, 0.39] \times [0.015, 0.025]$. It can be over-approximated by four abstract states, i.e., $[0, 0.2] \times [0, 0.02]$, $[0, 0.2] \times [0.02, 0.04]$, $[0.2, 0.4] \times [0, 0.02]$ and $[0.2, 0.4] \times [0.02, 0.04]$.

	\vspace{-1mm}
	\subsection{Model Checking of LTL Properties}
	\vspace{-1mm}
	
	\begin{wrapfigure}{R}{0.43\textwidth}
		\IncMargin{10pt}
		\vspace{-4mm}
		\begin{minipage}{0.43\textwidth}
			\setlength{\textfloatsep}{5pt}
			\begin{algorithm}[H]
				\caption{LTL Model Checking} 
				\label{alg:VbS}
				\KwInput{Initial abstract state $\mathbf{s}^0$, LTL formula $\phi$, threshold $T$, transition function $f$, Neural Network ${\cal N}$, perturbation $\epsilon$}
				\KwOutput{True, False} 
				$f$ = spot.formula.Not($\phi$)\\
				$A_f$ = spot.translate($f$)\\
				$K$ = spot.make\_kripke\_graph()\\
				$Q\leftarrow \{\mathbf{s}^0\}$\\
				$C = 1$\\
				\While{$Q$ is not empty and $C < T$}{
					Fetch $\mathbf{s}$ from $Q$\\
					\For{$i=1,\ldots,n$}{
						$[l_i,u_i]\leftarrow g(f(\mathbf{s}, {\cal N}(\mathbf{s})),i)$\\
						$[l_i,u_i] \leftarrow  [l_i - \epsilon_i, u_i + \epsilon_i]$
					}
					$\{\mathbf{s}^1,\ldots,\mathbf{s}^m\}:=h([l_{1}, r_{1}],\ldots, [l_{n}, r_{n}])$
					
					\For{$j=1,\ldots,m$}{
						
						\If{$\mathbf{s}^j$ is not traversed}{
							add\_edge($K$, $\mathbf{s}\rightarrow \mathbf{s}^j$)\\
							Push $\mathbf{s}^j$ into $Q$\\
							$C = C + 1$
						}
					}
				}
				calculate\_satisfied\_propositions($K$)\\
				\Return spot.intersecting\_run($K$, $A_f$)
			\end{algorithm}
			\vspace{-6mm}
		\end{minipage}
	\end{wrapfigure}
	Since we can construct the explicit finite-state transition system, the verification work can be delivered to existing model-checking tools. This observation indicates that the abstraction is decoupled from the subsequent verification procedure, which means that our approach can benefit from any future improvement in model-checking techniques.

	In practice, we leverage Spot~\cite{duret2016spot} to complete the verification work. Algorithm~\ref{alg:VbS} describes the implementation details of our verification framework, where \textbf{Input} lists the settings that users need to provide and functions that start with "spot." can be called directly from Spot. 
	
	$A_f$ is the automata that corresponds the negative form of the LTL formula $\phi$ (Line 1-2), where we refer readers to~\cite{DBLP:journals/fac/Konnov19} for more details of LTL verification. We traverse the abstract states via breadth-first search to build the explicit transition system $K$, where successive abstract states are computed in the way explained in Section~\ref{subsec:abs}. Function $f$ (Line 9) takes an interval vector and the corresponding action returned by DNN, and returns the irregular state space which we will not compute explicitly. Instead, we directly obtain the set of abstract sets after applying functions $g$ and $h$ (mentioned in Section~\ref{sub:soundness}). Besides, threshold $T$ will force to terminate  the verification when the model checker cannot verify all reachable abstract states.
	
	Then, we calculate the propositions satisfied by each abstract state in $K$. Note that for guaranteeing the soundness of verification results, when judging whether the abstract state satisfies proposition $\varphi$, we believe that it satisfies $\varphi$ only if all concrete states in it satisfy $\varphi$. Finally, we call the method in Spot to construct the transition diagram generated by $A_f$ and $K$ to obtain the verification result.
	
	\vspace{-1mm}
	\subsection{Soundness of the Abstraction Transformer}
	\label{sub:soundness}
	\vspace{-1mm}
	
	We prove that the abstraction transformer is sound in that it preserves propositions. 
	Let $\mathbb V$ be the set of interval vectors of ${S}_n$.
	The abstract transformer is a function $T^\#:{\cal S}_{n,d}\rightarrow 2^{{\cal S}_{n,d}}$, which is a composition of $g: {\cal S}_{n,d} \rightarrow {\mathbb V}$ and $h:{\mathbb V} \rightarrow 2^{{\cal S}_{n,d}}$. 
	Intuitively, $g(\mathbf{s})$ denotes the vector $V$ of intervals after the action determined by the neural network is applied to ${\mathbf{s}}$, and $h(V)$ returns the least set of abstract states whose union is an over approximation of an interval vector $V$. Note that for the brevity of proof, we omit the expansion operation for perturbation without loss of validity.
	
	\begin{lemma}\label{lemma:1}
		Given a state $s$ in ${S}_n$, let $T(s)$ denote the successor state after an action is applied to $s$ and  $\alpha(s)$ be the abstract state of $s$. Then, $T(s)\in g(\alpha(s))$. 
	\end{lemma}
	
	Lemma \ref{lemma:1} says that $g$ guarantees that after an action is applied to $\alpha(s)$, 
	it generates an interval vector that contains the successor state $T(s)$ caused by applying the same action to $s$. 
	
	\begin{lemma}\label{lemma:2}
		Given a vector $V$ and a state $s$ such that  $s\in V$, let  $\alpha(s)$ be the abstract state of $s$. 
		Then, $\alpha(s)\in h(V)$.  
	\end{lemma}

	\begin{wrapfigure}{r}{0.45\textwidth}
		\vspace{-8mm}
		\setlength{\textfloatsep}{5pt}
		\begin{center}
			\includegraphics[width=0.35\textwidth]{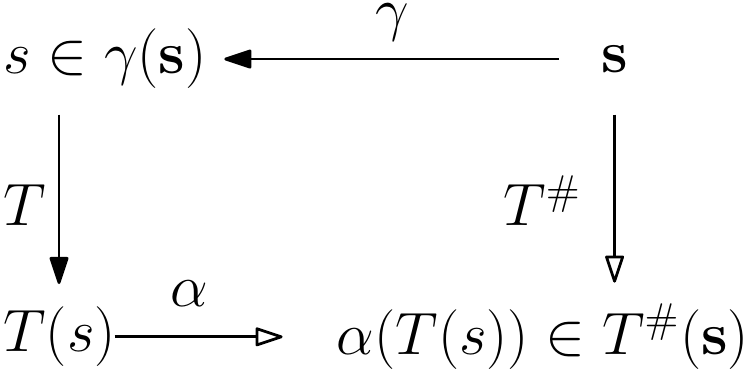}
	\caption{Soundness of abstract transformer}
		\end{center}
		\vspace{-5mm}
		\label{fig:sound}
	\end{wrapfigure}

	Lemma \ref{lemma:2} guarantees over-approximation of interval vectors. That is, for each state $s$ that is contained in $V$, the abstract state of $s$, i.e., $\alpha(s)$, must be in $h(V)$. The formal definitions of $g$ and $h$ and proofs are provided in the appendix as supplementary document.

	Figure \ref{fig:sound} graphically shows the soundness of the abstract transformer $T^\#$. It says that for any abstract state ${\mathbf{s}}$, the transitions from ${\mathbf{s}}$ to its successor abstract states in $T^\#({\mathbf{s}})$ cover all the transitions from the concrete states that ${\mathbf{s}}$ abstracts to their successor states that are caused by the same action. 
	
	\begin{theorem}[Soundness]
		For each ${\mathbf{s}} \in {\cal S}_{n,d}$, $\alpha(T(s))\in T^\#({\mathbf{s}})$ holds for all the states $s$ that ${\mathbf{s}}$ abstracts. 
	\end{theorem}
	\begin{proof}
		The proof is straightforward with Lemmas \ref{lemma:1} and \ref{lemma:2}. Let $s$ be an arbitrary state that is abstracted by ${\mathbf{s}}$, i.e., $s\in \gamma({\mathbf{s}})$. 
		By Lemma \ref{lemma:1}, there is $T(s)\in g(\alpha(s))$. Let $s'=T(s)$ and $V=g(\alpha(s))$, 
		Because $s'\in V$, we have $\alpha(s')\in h(V)$ according to Lemma \ref{lemma:2}. 
		Namely, $\alpha(T(s))\in h(g(\alpha(s)))$. Because ${\mathbf{s}} = \alpha(s)$ and $T^\#=h \circ g$, 
		we conclude $\alpha(T(s))\in T^\#({\mathbf{s}})$.
	\end{proof}

	\vspace{-1mm}
	\section{Experimental Evaluation}
	\vspace{-1mm}
	\label{sec:experiment}
	
	We first study the impact of abstract granularity on the performance of trained systems by training a system under different abstract granularities and comparing their performance. Then, we demonstrate the effectiveness of our approach by showing that the systems trained in our approach have comparable performance with whose trained in classical DRL algorithms. Finally, we verify the trained systems against their desirable properties to show the efficiency of the verification.  
	
	\vspace{-1mm}
	\subsection{Benchmark and Experimental Settings}
	\vspace{-1mm}
	
	We choose three classic control problems from Gym~\cite{1606.01540}, including Pendulum, Mountain Car and Cartpole, and another adapted control task 4-Car Platoon~\cite{DBLP:conf/pldi/ZhuXMJ19}. 
	
	\begin{itemize}
		
		\item \textbf{Pendulum} It delineates a pendulum that can rotate around an endpoint. By starting from a random position, a pendulum is expected to swing up and stay upright. The expected property of Pendulum is that its angle must be always in the preset range.
		
		\item \textbf{Mountain Car} A car is positioned on a one-dimensional track between two mountains. It is expected to drive up the right mountain by first driving to the left one to get enough power via inertia after training. We need to guarantee that the car can finally reach the destination.
		
		\item \textbf{Cartpole} A pole is attached by an un-actuated joint to a cart, which moves along a frictionless track. The controller aims to keep the angle of the pole and the displacement of the cart within fixed thresholds, which must be guaranteed to satisfy.
		
		\item \textbf{4-Car Platoon} Four cars on the road are supposed to drive in a platoon behind each other. Each car aims to drive close to the front car so as to save fuel and reduce driving time. A straightforward safety requirement is that the four cars must never cause any collision.
		
	\end{itemize}
	
	\paragraph{Experimental settings} All experiments are conducted on a workstation running Ubuntu 18.04 with a 32-core AMD Ryzen Threadripper CPU @ 3.7GHz and 128GB RAM.

	\subsection{Impact of Abstraction Granularity}
	\label{sub:ag}

\begin{table}
		\vspace{-6mm}
			\footnotesize
		\caption{Values of abstraction granularity}
		\label{tab:ag}
		\centering
		
			\begin{tabular}{l|ccc}
				\hline
				\textbf{Controller} & \textbf{Fine} & \textbf{Intermediate} & \textbf{Coarse} \\
				\hline\hline
				Pendulum  & $[\{10^{-3}\} \times 3]$ & $[10^{-3}, 10^{-3}, 10^{-2}]$  & $[\{10^{-2}\} \times 3]$ \\
				
				Mountain Car & $[10^{-5}, 10^{-6}]$ & $[10^{-3}, 10^{-3}]$  & $[10^{-2}, 10^{-3}]$\\	
				
				CartPole  & $[\{10^{-3}\} \times 4]$ & $[10^{-2}, 10^{-2}, 10^{-3}, 10^{-2}]$ & $[\{10^{-2}\} \times 4]$\\
				
				4-Car platoon  & $[10^{-2}, \{10^{-3}, 10^{-2}\} \times 3]$  & $[\{10^{-2}\} \times 7]$ &$[\{20^{-2}\} \times 7]$\\
				\hline
		\end{tabular}
	\end{table}
	
	We trained the four systems in the abstraction-based approach. To evaluate the impact of abstraction granularity, we set three different abstraction granularity values for each system, and examine their performance. Table ~\ref{tab:ag} shows the values of the four systems. We use $\{x\} \times y$ to indicate that there are $y$ consecutive $x$s in the vector for convenience.  Smaller intervals imply finer abstract granularity. We classify the abstraction granularity into three levels, i.e., \textit{fine}, \textit{intermediate} and \textit{coarse}. 
	
	 To compute a precise result, we train each controller for 10 rounds and record its performance at corresponding steps. The performance of a system is measured by the average reward value based on 5 episodes. To make the comparison clear, we omit confidence intervals and only show the trend of mean rewards in the figure. 
	
	Figure \ref{fig:ag} shows the comparison results of the four trained systems. It can be seen that the performances of controllers in Pendulum are almost the same under the three different abstract granularity. It implies that even the coarse one is enough to train the system with a good performance. 
	However, the performance of MountainCar and Cartpole varies with abstract granularities.  Finer abstract granularity leads to a better performance. Note that the trajectories of  the 4-Car platoon's performance fluctuate heavily because the controller will receive a big negative reward when the cars collide according to the reward setting. 
	
The experiments also show that it is important to choose an appropriate abstraction granularity to achieve a trade-off between the performance of trained systems and the size of their abstract state spaces. One way to determine  is that we can start training a controller with the relatively coarse abstraction granularity, and refine it until the controller converges steadily to the optimal reward that is indicated by the controller trained with classic DRL algorithms.

	\subsection{Performance Comparison with Classical DRL Algorithms}
	\label{sub:pe}

		\begin{figure}[tb]
		\centering
		\includegraphics[width=1.0\columnwidth]{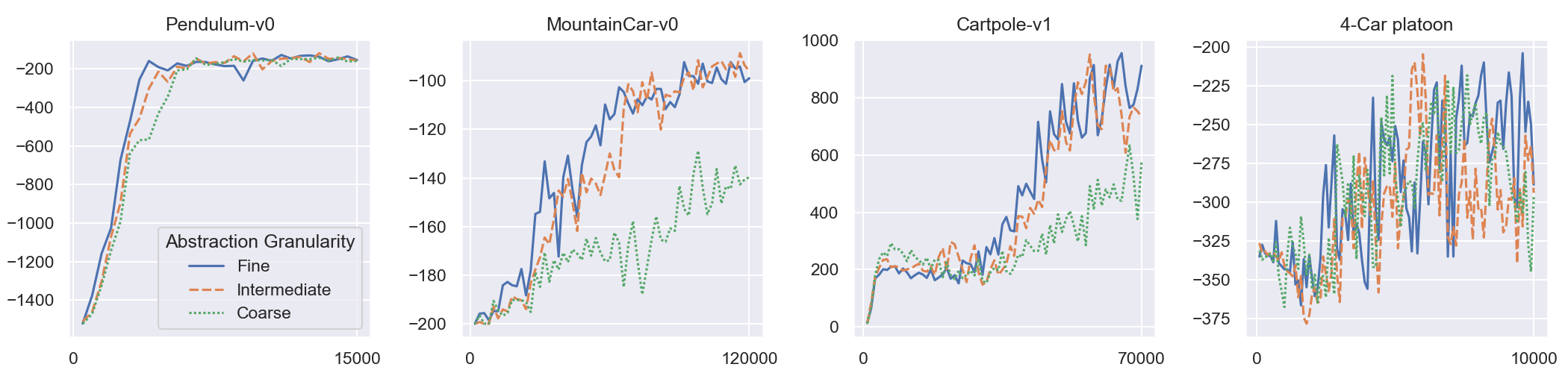}
		\vspace{-4mm}
		\caption{Performance comparison under different abstraction granularity}
		\vspace{-3mm}
		\label{fig:ag}
	\end{figure}

		\begin{wraptable}{r}{0.55\textwidth}
		\vspace{-5mm}
		\centering
		\setlength{\tabcolsep}{1pt}
		\footnotesize
		\caption{Basic settings of four systems for training}
		\label{tab:train-setting}
		\begin{tabular}{l|ccc}
			\hline
			\textbf{Controller} & \textbf{Network} & \textbf{Algorithm} &  \textbf{Granularity} \\
			\hline
			
			Pendulum  & $400\times300$ & DDPG  & $[\{10^{-3}\} \times 3]$\\
			
			Mountain Car & $32\times32$ & DQL  & $[10^{-5}, 10^{-6}]$\\	
			CartPole & $32\times32$ & DQL   & $[\{10^{-3}\} \times 4]$ \\
			4-Car platoon  & $500\times400$ & DDPG  & $[10^{-2}, \{10^{-3}, 10^{-2}\} \times 3]$  \\
			
			\hline
		\end{tabular}
\vspace{-1mm}
	\end{wraptable}

	We compare the performance of our training approach with classical DRL algorithms. 
We train each control system using a classic DRL algorithm and its corresponding extension with our abstraction technique, respectively. The main training settings can be found in Table~\ref{tab:train-setting}. For those remaining adjustable hyperparameters, we use the default values in the TF2RL framework~\cite{ota2020tf2rl}.

	\begin{figure}[t]
		\centering
		\includegraphics[width=1.0\columnwidth]{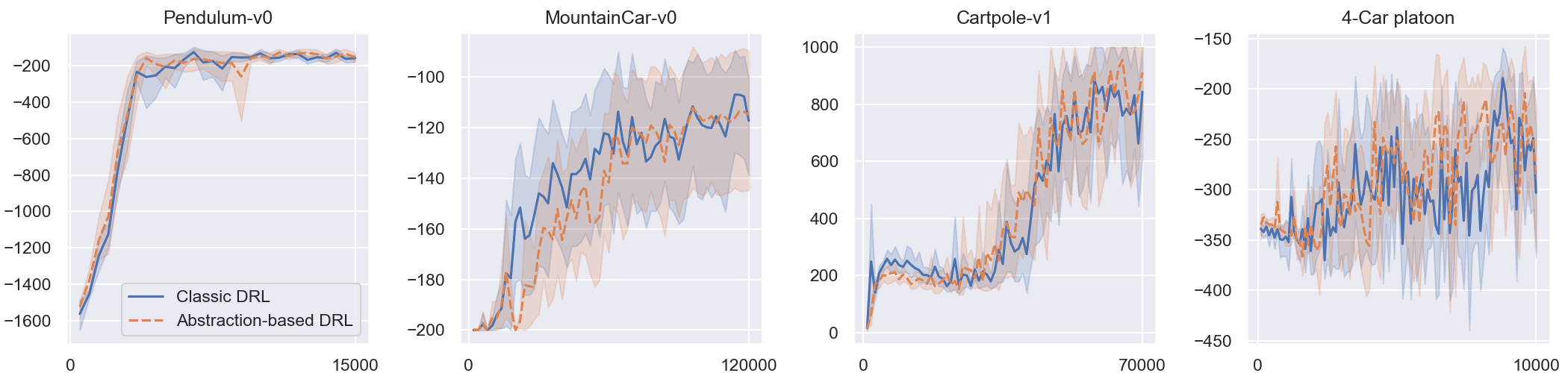}
		\vspace{-2mm}
		\caption{Performance comparison of the four systems with and without abstraction}
		\vspace{-2mm}
		\label{fig:comparison}
	\end{figure}

	Figure~\ref{fig:comparison} depicts the trend of four controllers' performance as the training proceeds under different training frameworks. The blue line indicates the mean rewards of a controller trained with the classic DRL algorithm. The light blue area shows the corresponding confidence intervals. The orange dashed line and area represent the performance of the controllers that are trained with the same DRL algorithms but extended with our abstraction approach. It can be observed that the trends of mean rewards are similar in all the four cases.   
	Although in Mountain Car and 4-Car platoon, there is a performance gap during the training process, the controller trained with abstraction can achieve the optimal reward eventually. 
	Thus, the controllers trained by the abstraction technique can retain comparable performance against those trained without abstraction.
	
	\subsection{Verification Analysis}
	
	In this section, we model check the four controllers that are trained in our abstraction-based approach and present the verification results. Table~\ref{tab:vr} shows the experimental data, where $\epsilon$ represents the perturbation vector,  \textbf{All} indicates whether the abstract state space is completely verified, \textbf{States} means the number of traversed abstract states, \textbf{Verified} indicates whether the property is verified true $(\checkmark)$ or false (\ding{55}), and \textbf{Time} denotes the time cost in second. We preset a threshold of the number of traversed abstract states $1.0 \times 10^7$ to force the verification to terminate.
	\vspace{-2mm}
	
	\paragraph{Pendulum} 
	One property of Pendulum is that the pole's angle must always be in $\pm 90\si{\degree}$. We use a tuple $(\theta,\omega)$ to define a state of Pendulum, where $\theta$ and $\omega$ denote the pole's angle and angular velocity, respectively. Then the property can be defined by the LTL formula $G( |\theta(s)|\leq \frac{\pi}{2} )$, where $G$ is the global operator indicating that the proposition following $G$ must hold in all the reachable states. We assume the initial state space of the controller is $[0,0.01] \times [0,0.01]$. The property can be verified under different perturbations in several seconds. The other property is that the angular velocity must be greater than $0$ eventually if the angle is less than $0$, which is not verified. Due to the over-approximation of reachable concrete states, the set of successor abstract states of the specific one in the violated path not only contains the valid abstract state, but also include the one with angular velocity less than $0$, which means that not all paths satisfy the property. 
	
	\vspace{-2mm}
	\paragraph{Mountain Car}  We use $p(s)$ and $v(s)$ to represent the position and velocity of the state in Mountain Car. There are two properties to ensure that the car can eventually reach the destination. The first says that the speed of car must be greater than $0.02$ around position $0.2$, which is represented by the LTL formula $G(|p(s)-0.2| < 0.05 \rightarrow v(s) > 0.02)$. The other property is that the car can always reach the position $0.5$, which can be formulated as $F(p(s) \geq 0.5)$, where $F$ is the finally operator in LTL. The initial position of the car is set $[-0.5001, -0.5]$. Both the two properties can be verified to be true under different perturbations.

	\vspace{-2mm}
	\paragraph{Cartpole} 
	One property of Cartpole  is that the angle of the pole and the displacement of the cart should never exceed preset thresholds. We assume the thresholds are 2.4 and $\frac{\pi}{15}$, respectively. The property can be defined as $G(|\theta(s)| \leq 2.4 \land |p(s)| \leq \frac{\pi}{15})$, where $\theta(s)$ and $p(s)$ represent the angle and displacement respectively. It is partially verified to be true on $1.0 \times 10^{7}$ abstract states in 4.3 hours, where $1.0 \times 10^{7}$ is the threshold we set for the number of traversed abstract states.
	
	\vspace{-2mm}
	\paragraph{4-Car Platoon} One safety property of the system is that there must be no collisions between cars. 
	We use $d_{i}(s)$ to denote the distance between the $i$-th car and $(i+1)$-th car,  the property can be formulated as $G(d_{1}(s) >= 0\land d_{2}(s) >= 0 \land d_{3}(s) >= 0)$. Due to the large state space, the property is verified to be true on a partial number of states in nearly 6.0 hours. 
	
	\begin{table}
		\caption{Verification results of the four trained controllers}
		\centering
		\setlength{\tabcolsep}{1pt}
					\footnotesize 
			\begin{tabular}{l|c|clcccr}
				\toprule
				\textbf{Case} & \textbf{Initial State Space} & \textbf{Property} & \textbf{$\epsilon$} &  \textbf{All} & \textbf{States} & \textbf{Verified}& \textbf{Time}\\
				\hline \hline  
				
				\multirow{4}{*}{Pendulum} &   &   \multirow{3}{*}{$G(|\theta(s)| \leq \frac{\pi}{2})$}  &  $[0,0]$   & $\checkmark$ & $1.1\times10^{3}$ &  $\checkmark$  & 2 \\
				& \multirow{2}{*}{$[0,0.01]\times[0,0.01]$}  &  &   $[0,0.01]$  & $\checkmark$ & $1.2\times10^{3}$ &$\checkmark$& 2\\
				&  & & $[0,0.1]$ & $\checkmark$ & $1.0\times10^{4}$  &  $\checkmark$  & 16 \\
				\cline{3-8}
				&   &   $G(\theta(s) \leq 0 \rightarrow F(\omega(s) \geq 0))$  &  $[0,0]$   & $\checkmark$ & $1.1\times10^{3}$ &  \ding{55} & 2 \\
				
				\hline
				
				\multirow{6}{1.3cm}{Mountain Car}   & 	\multirow{6}{*}{$[-0.5001,-0.5]\times[0,0]$}  & \multirow{3}{3.1cm}{$G(|p(s)-0.2|<0.05\rightarrow v(s) > 0.02 )$}   &  $[0,0]$   & $\checkmark$ & $1.3\times10^6$ & $\checkmark$ & 1164 \\
				&    &    &  $[10^{-5},0]$  &  $\checkmark$  &  $1.7 \times 10^6$  &  $\checkmark$  &  1555 \\
				&    &    &  $[10^{-4},0]$  &  $\checkmark$  &  $3.9 \times 10^6$  &  $\checkmark$  &  3717\\
				\cline{3-8}
				&   & \multirow{3}{*}{$F(p(s) \geq 0.5)$} &  $[0,0]$  & $\checkmark$ & $1.3\times 10^6$ &  $\checkmark$  & 1154\\
				&  &  &  $[10^{-5},0]$   & $\checkmark$ &  $1.7 \times 10^6$  &  $\checkmark$  & 1542 \\
				&  &  &  $[10^{-4},0]$   & $\checkmark$ &  $3.9 \times 10^6$  &  $\checkmark$  & 3742  \\	
				
				\hline	
				
				\multirow{2}{*}{Cartpole}    & $[0,0.01] \times [0,0.01] \times$ & \multirow{2}{*}{$G(|\theta(s)| \leq 2.4\land |d(s)| \leq \frac{\pi}{15})$} &  $[\{0\} \times 4]$  & \ding{55} &   $1.0 \times 10^7$  &  $\checkmark$  & 14869 \\
				& $[0,0.01] \times [0,0.01]$ &   & $[\{10^{-3}\} \times 4]$ & \ding{55} &  $1.0 \times 10^7$  & $\checkmark$ & 15443\\
				
				\hline	
				
				\multirow{3}{1.3cm}{4-Car Platoon} & $[0,0.1]\times[0,0.01]\times[0,0.1] $ & $G(d_{1}(s) >= 0  $   & \multirow{3}{*}{ $[\{0\}\times7]$ } &  \multirow{3}{*}{\ding{55}}  & \multirow{3}{*}{ $1.0 \times 10^7$}  & \multirow{3}{*}{ $\checkmark$}  & \multirow{3}{*}{ 21562}  \\
				& $\times [0,0.01]\times[0,0.1] $   &  $\land d_{2}(s) >= 0$  &   &   &   &  &   \\
				& $\times[0,0.01]\times[0,0.1] $   &  $\land d_{3}(s) >= 0)$  &    &    &   &    &  \\

				\bottomrule

		\end{tabular}
		\label{tab:vr}
	\end{table}
	
	\vspace{-2mm}
	
	\paragraph{Efficiency and Scalability} The experimental results show that the time cost on verification mainly depends
	on the size of reachable abstract-state space. It is possible to verify a DRL system which has millions of abstract states in a few hours. For the systems that have larger abstract-state space, we can fine-tune the abstraction granularity to reduce the abstract-state space during the training phase and meanwhile we guarantee that all the desired properties must be verified to be true under that granularity. 
	A case study in the supplementary document shows the feasibility of the approach. 
	 
	
		\vspace{-1mm}
	
	\section{Related Work}
	\label{sec:rw}
		\vspace{-1mm}
	

	Our abstraction-based verification approach is inspired by a bunch of recently emerging works on abstraction-based verification of neural networks~\cite{singh2019abstract,pulina2010abstraction,prabhakar2019abstraction}. These works have demonstrated the effectiveness of abstraction techniques on formal verification of neural networks. By contrast, there are two major differences in our abstraction approach. One is that we introduce abstraction in the training process, and the other is that the abstraction objects in our framework are system states, while the abstract objects in these approaches are neurons.

	To the best of our knowledge, existing verification approaches for DRL-enabled systems can be divided into three categories.  
	One is based on model transformation, which transforms the embedded DNN model into an interpretable model such as decision trees and programs~\cite{bastani2018verifiable,DBLP:conf/icml/VermaMSKC18}. 
	Another is to synthesize barrier functions that assist the DNN in decision making can ensure safety during deployment~\cite{DBLP:conf/pldi/ZhuXMJ19,xiong2021scalable}. 
	The last is to incorporate the DNN into the system dynamics~\cite{DBLP:conf/hybrid/IvanovWAPL19,tran2019safety}. 
	However, these approaches fail to simultaneously fulfill the three key functionalities that our approach achieves, i.e., the scalability that is impervious to the size of DNN, supporting more complex temporal  properties other than safety, and the capability of dealing with perturbations in verification.

	
	Runtime verification is another perspective of applying formal methods to DRL-enabled systems~\cite{DBLP:journals/jmlr/GarciaF15}. 
	For instance, runtime monitoring based on formal methods can guide agents to behave under predefined requirements~\cite{DBLP:conf/aiia/HasanbeigKA19,DBLP:conf/aaai/FultonP18}. 
	Agents can be prevented from unsafe behaviors by constructing the safety shield via formal verification~\cite{DBLP:journals/corr/abs-1708-08611,jansen2018shielded}. However, runtime approaches  inevitably incur extra system overheads during training and deployment.

	\vspace{-1mm}
	\section{Conclusion and Future Work}
	\label{sec:c}
		\vspace{-1mm}
		
	We have presented an abstraction technique for training and verifying DRL-enabled systems. By abstraction, continuous state space is discretized into a finite set of abstract states, on which DNNs in control systems are trained. By the same abstraction, we can model the trained DRL-enabled systems by finite state transition systems and resort to state-of-the-art model checking techniques to verify various system properties.  We  proved the soundness of the abstraction in verification and implemented a training and verification framework. We conducted experiments on four classical control problems. The experimental results demonstrated that the controllers trained with abstraction have comparable performance with those trained without abstraction. The four trained controllers were verified, which showed  the feasibility and efficiency of our verification approach. 
	
	We believe that learning on abstract domain would be a promising technique for training verifiable AI-enabled systems. It provides a flexible mechanism of achieving a balance between performance of trained systems and the size of abstract state space by fine-tuning abstraction granularity. Based on this work, it would be possible to integrate training and verification techniques. Guided by the counterexamples generated during verification, the learning process can be more strategic to produce well-trained DRL systems with formal guarantees. 

	\newpage 
	\appendix

\section*{Appendix}
	 The appendix consists of two sections. In Section A, we present an experiment of achieving the balance between the performance of trained systems and the size of their abstract-state space by adjusting its abstraction granularity during learning. The result can support our claim in our paper that \textit{by fine-tuning abstraction granularity we can reduce the abstract-state space and meanwhile guarantee that all the desired properties are verifiable}. Section B gives the proofs of the lemmas that are used to prove the soundness of our abstraction in verification.

    \section{Fine-tuning Abstraction Granularity}
	\label{sub:determine-d}
	
	In this section, we use Cartpole as a supplementary example for Section 5.4 to illustrate the feasibility of our approach for achieving a balance between the performance of trained systems and the size of their abstract-state space by fine-tuning the abstraction granularity. The approach makes the verification of trained systems amenable without losing their performance. 
	
	\subsection{Determining Abstraction Granularity Baseline}

	\begin{wrapfigure}{r}{0.5\textwidth}
	    \label{fig:cp}
		\vspace{-3mm}
		\setlength{\textfloatsep}{5pt}
		\centering
		\includegraphics[width=0.5\textwidth]{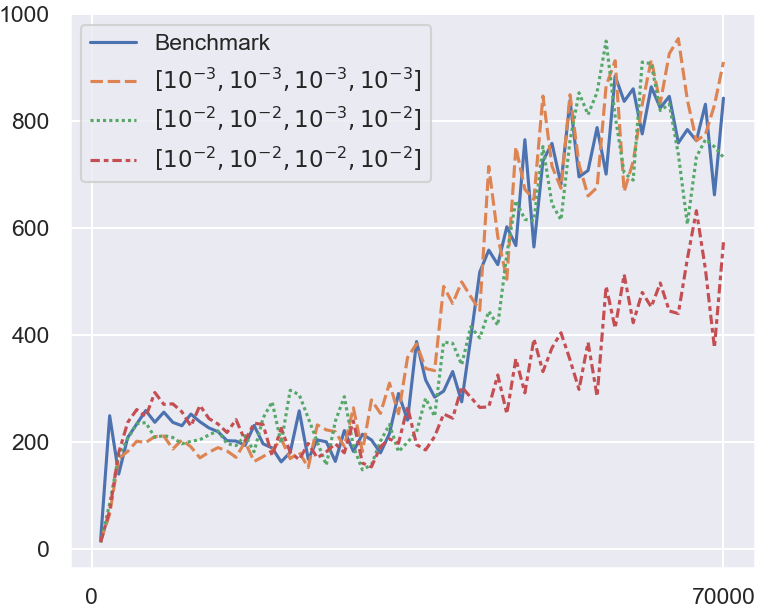}
		\vspace{-4mm}
	    \caption{Performance comparison of the trained controller under different abstraction granularities}
		\vspace{-4mm}
	\end{wrapfigure}
	
	In the example of Cartpole, a state consists of four elements, i.e., the displacement and the velocity of the cart and the angle and the angular velocity of the pole. Because states are continuous, the whole state space is infinite. The controller aims to keep the angle of the pole and the displacement of the cart within fixed thresholds.

	As we have discussed in our submitted paper, we can start training the example with the relatively coarse abstraction granularity, e.g., $[10^{-2}, 10^{-2}, 10^{-2}, 10^{-2}]$, but observed that there is a big gap in the performance when we compare the trained controller under this granularity with the one trained using classical DRL algorithm without any abstraction. We call the latter benchmark. 
	We fine-tune the granularity until the controller converges steadily to the rewards that are indicated by the benchmark controller. As shown in Figure 7, the performance of the intermediate abstraction granularity ($[10^{-2}, 10^{-2}, 10^{-3}, 10^{-2}]$) is close to the benchmark, and therefore we can choose it as the baseline of the abstraction granularity. 
	An abstraction granularity that is finer than the baseline basically guarantees that the controller's performance is close to the benchmark.

	\subsection{Model Checking under Different Abstraction Granularities}


	            
				
				

	
	The property for verification can be formally defined by the following LTL formula: 
	\begin{align}
	G(|p(s)| \leq 2.4 \land |\theta(s)| \leq \frac{\pi}{15}),\notag 
	\end{align} 
	where, $p(s)$ and $\theta(s)$ represent the displacement and the angle respectively. 
	There are three possible verification results, i.e., completely verified to be true ($\checkmark$), partially verified to be true ($-$), or 
	verified to be false (\ding{55}). Table~\ref{tab:ag} shows the verification results of the property under different abstraction granularities. 
	It can be seen that the property fails in the first three cases. The verification framework found  some abstract states where the property is violated before reaching the threshold. 
	Note that finding violations does not imply the trained controllers do not satisfy the property. 
	That is a common phenomena caused by introducing abstraction into verification.  We continued to refine the abstraction granularity to be $[10^{-3}, 10^{-3}, 10^{-3}, 10^{-3}]$. 
	The property is verified to be true on $1.0\times 10^6$ abstract states before it reaches the state-space threshold. This result is the same as one shown in the paper.

	To better illustrate how we can achieve a complete verification of a desired property, we slightly relaxed the safety property by increasing the safe range of the pole's angle by 0.1. That is, we replace $\frac{\pi}{15}$ (approximately $0.21$) in the property with $0.22$. 
	The relaxed safety property becomes:
		\begin{align}
	G(|p(s)| \leq 2.4 \land |\theta(s)| \leq 0.22).\notag 
	\end{align} 
	The right-hand part shows the verification results of the property under different granularities. 
	In the case of the coarse granularity, a violation to the property is found, indicating that the property may not be satisfied by the controller trained under this abstraction. When we fine-tuned the granularity to be $[10^{-3}, 10^{-2}, 10^{-3}, 10^{-2}]$, the relaxed property was completely verified within a reasonable time. Because this granularity is finer than the baseline and the property is verified, we can safely stop training with finer granularities in practice. 
	In this experiment, we continued to train the controller under two finer granularities for the purpose of comparison. It shows that the property is preserved as expected under finer granularities. In the case under $[10^{-3}, 10^{-3}, 10^{-3}, 10^{-3}]$, 
	the property was partially verified due to reaching the threshold of traversed abstract states. 
	
	

	\begin{table}
		\vspace{-6mm}
		\centering
		\caption{Verification results with different properties on the angle and abstraction granularities. (The state threshold of the number of traversed abstract states is $1.0 \times 10^6$, and time is in second.)}
	
		\setlength{\tabcolsep}{5pt}
			\begin{tabular}{cc|ccr|ccr}
\toprule

                 \multicolumn{2}{c|}{\multirow{2}{*}{\textbf{Granularity}}} & \multicolumn{3}{c|}{\textbf{$G(|p(s)| \leq 2.4 \land |\theta(s)| \leq \frac{\pi}{15})$}} & \multicolumn{3}{c}{\textbf{$G(|p(s)| \leq 2.4 \land |\theta(s)| \leq 0.22)$}} \\
                        \cline{3-8}
                        & & \textbf{Result} & \textbf{States} & \textbf{Time} & \textbf{Result} & \textbf{States} & \textbf{Time}   \\
\hline 
Coarse & $[10^{-2}, 10^{-2}, 10^{-3}, 10^{-2}]$ &  \ding{55} & $4.5\times10^{5}$ & 697  & \ding{55} & $4.5\times10^{5}$ &  710\\
\multirow{2}{*}{$\big\downarrow$} &$[10^{-3}, 10^{-2}, 10^{-3}, 10^{-2}]$  &   \ding{55} & $5.1\times10^{5}$ & 799   &  \checkmark & $5.1\times10^{5}$& 793\\
 & $[10^{-3}, 10^{-3}, 10^{-3}, 10^{-2}]$&  \ding{55} &$3.5\times10^{5}$ & 462  & \checkmark  & $3.5\times10^{5}$ & 456 \\
Fine & $[10^{-3}, 10^{-3}, 10^{-3}, 10^{-3}]$    &   - & $1.0\times10^{6}$ & 1361 &  - &  $1.0\times10^{6}$ & 1347 \\
 				\bottomrule
				
		\end{tabular}
		\label{tab:ag}

        \vspace{-1mm}
	\end{table}

	\section{Proof of Lemmas in Section 4.3}
	\label{sub:proof}

	We first formally define concretization function $\gamma$, abstraction function $\alpha$,  concrete transformer $T$,  and abstract transformer $T^\#$ in sequence.

	\begin{definition}[Concretization function]
		A concrete function $\gamma: {\cal S}_{n,d}\rightarrow 2^{{S}_n}$ maps abstract states into sets of concrete system states , where
		$\gamma((I_1,\ldots,I_n))=\{s : s_i\in [l_i,u_i], \forall i \in [1, n] \}$.
	\end{definition}
	
	\begin{definition}[Abstraction function]
		An abstraction function $\alpha: S_n \rightarrow {\cal S}_{n,d}$ is the mapping from concrete domain to abstraction domain such that for all $s\in {S}_n$ there is $\alpha(s)=(I_1,\ldots,I_n) \in {\cal S}_{n,d}$ and $s_i\in [l_i,u_i]$ for each $I_i = [l_i, u_i]$.
	\end{definition}
	
	\begin{definition}[Concrete transformer]
		A concrete transformer $T: {S}_n \rightarrow {S}_n$ defines the state transition function based on the system dynamics.
	\end{definition}
	
	\begin{definition}[Abstract transformer]
		An abstract transformer $T^\#: {\cal S}_{n,d} \rightarrow 2^{{\cal S}_{n,d}} $ defines transitions between abstract states, which is a composition of an interval transition function $g: {\cal S}_{n,d} \rightarrow {\mathbb V}$ and an over-approximation function $h:{\mathbb V}\rightarrow 2^{{\cal S}_{n,d}}$. 
	\end{definition}
	
	Intuitively, $g({\mathbf{s}})$ denotes the interval vector $V$ after the action determined by the neural network is applied to ${\mathbf{s}}$, and $h(V)$ returns the least set of abstract states whose union is an over approximation of the interval vector $V$.
	
	We use the tuple $(c_1,\ldots,c_n)$ to denote values of the concrete state $s$ and $f_i(s, a)$ to denote the variation in each variable based on the system dynamics. Specifically, given the concrete state $s$ and the adopted action $a$, the next state $s^{'}$ of $s$ can be represented by $(c_1 + f_1(s,a),\ldots,c_n + f_n(s,a))$.

	\begin{definition}[Interval transition function]
		An interval transition function $g: {\cal S}_{n,d} \rightarrow \mathbb V$ returns an interval vector $V$ for a abstract state ${\mathbf{s}}$. Let ${\mathbf{s}} = (I_1,\ldots,I_n)$ and $V = (I_1^{'},\ldots,I_n^{'})$. Then for each $I_i^{'} = [l_i^{'}, u_i^{'}]$, $l_i^{'} = l_i + min\{f_i(s,a): s \in {\mathbf{s}}\}$ and $u_i^{'} = u_i + max\{f_i(s,a): s \in {\mathbf{s}}\}$. 
	\end{definition}
	
	\begin{definition}[Over-approximation function]
		An over-approximation function $h:{\mathbb V}\rightarrow 2^{{\cal S}_{n,d}}$ returns a set of abstract states for a given interval vector $V$. Let the interval vector $V^{'}$ denote the range of the union of abstract states. Then $\forall i \in [1, n]$, $l_i^{'} \leq l_i \leq u_i \leq u_i^{'}$ and $d_{i} | (u_i^{'} - l_i^{'})$. Let $m_i = (u_i^{'} - l_i^{'})/{d_{i}}$, there are $\Pi^{n}_{i=1}m_i$ abstract states in the set.
	\end{definition}

	\begin{proof}[Lemma 1]
		Let the state $s = (c_1,\ldots,c_n)$. Then $T(s) = (c_1 + f_1(s, a),\ldots,c_n + f_n(s,a)) = (c_1^{'},\ldots,c_n^{'})$. $\alpha(s)$ is the abstract state that can be represented by $(I_1,\ldots,I_n)$, where $c_i \in [l_i, u_i]$ for each $I_i = [l_i, u_i]$. $g(\alpha(s))$ is the interval vector $(I_1^{'},\ldots,I_n^{'})$. Since the actions adopted for states belonging to the same abstract state are same, for each $I_i^{'} = [l_i^{'}, u_i^{'}]$, we have $l_i^{'} = l_i + min\{f_i(s^{'},a): s^{'} \in \alpha(s)\}$ and $u_i^{'} = u_i + max\{f_i(s^{'},a): s^{'} \in \alpha(s)\}$. Intuitively, $\forall i \in [1, n]$, $l_i \leq c_i \leq u_i$ and $min\{f_i(s^{'},a): s^{'} \in \alpha(s)\} \leq f_i(s, a) \leq  max\{f_i(s^{'},a): s^{'} \in \alpha(s)\}$, so we have $l_i^{'} \leq c_i + f_i(s, a) \leq u_i^{'}$, i.e., $l_i^{'} \leq c_i^{'} \leq u_i^{'}$. So $T(s) \in g(\alpha(s))$.  
	\end{proof}
	
	We prove the Lemma 2 via proof by contradiction. That is, we assume $\alpha(s) \notin h(V)$ and deduce the contradiction.
	
	\begin{proof}[Lemma 2]
		Let the state $s=(c_1,\ldots,c_n)$, the vector $V = (I_1,\ldots,I_n)$ and $h(V) = (I_1^{'},\ldots,I_n^{'})$. According to the definition of $h$, $\forall i \in [1, n]$, we have $l_i^{'} \leq l_i$ and $u_i \leq u_i^{'}$. Let $\alpha(s) = (I_1^{s},\ldots,I_n^{s})$. Since $\alpha(s) \notin h(V)$, we have $u_i^{s} < l_i^{'}$ or $l_i^{s} > u_i^{'}$ for each $I_i^{s} = [l_i^{s}, u_i^{s}]$. However, we have $l_i \leq c_i \leq r_i$ for each $i \in [1, n]$, therefore $c_i \notin [l_i^{s}, r_i^{s}]$, which contradicts the definition of $\alpha$. 
	\end{proof}

	\end{document}